\documentclass[11pt,onecolumn]{IEEEtran}
\usepackage{stfloats}
\usepackage{comment}
\usepackage{cite,booktabs}
\usepackage{graphicx}
\usepackage[cmex10]{amsmath}
\usepackage{algorithmic}
\usepackage{algorithm}
\usepackage{array}
\usepackage{mdwmath}
\usepackage{mdwtab}
\usepackage[caption=false,font=footnotesize]{subfig}
\usepackage{fixltx2e}
\usepackage{url}
\usepackage{xcolor}
\usepackage{amssymb}
\usepackage{mathtools}

\usepackage{amsthm}
\usepackage{bm}

\newtheoremstyle{mystyle}
  {}
  {}
  {\itshape\bfseries}
  {}
  {\bfseries}
  {.}
  { }
  {}

\newtheoremstyle{mystyle2}
  {}
  {}
  {}
  {}
  {\bfseries}
  {.}
  { }
  {}

\theoremstyle{mystyle}

\newtheorem{theorem}{Theorem}
\newtheorem{lemma}[theorem]{Lemma}

\theoremstyle{mystyle2}
\newtheorem{definition}[theorem]{Definition}

\DeclareMathOperator{\Tr}{Tr}
\DeclareMathOperator{\vect}{vec}

\begin{document}
\title{Dictionary and image recovery from incomplete and random measurements}

\author{Mohammad~Aghagolzadeh, 
        Hayder~Radha
\thanks{M. Aghagolzadeh and H. Radha are with the Department
of Electrical and Computer Engineering, Michigan State University, East Lansing,
MI, 48823 USA e-mail: \{aghagol1, radha@msu.edu\}.}}


\maketitle

\begin{abstract}

This paper tackles algorithmic and theoretical aspects of dictionary learning from incomplete and random block-wise image measurements and the performance of the \textit{adaptive dictionary} for sparse image recovery. This problem is related to \textit{blind compressed sensing} in which the sparsifying dictionary or basis is viewed as an unknown variable and subject to estimation during sparse recovery. However, unlike existing guarantees for a successful blind compressed sensing, our results do not rely on additional structural constraints on the learned dictionary or the measured signal. In particular, we rely on the \textit{spatial diversity} of compressive measurements to guarantee that the solution is unique with a high probability. Moreover, our distinguishing goal is to measure and reduce the estimation error with respect to the \textit{ideal dictionary} that is based on the complete image. Using recent results from random matrix theory, we show that applying a slightly modified dictionary learning algorithm over compressive measurements results in accurate estimation of the ideal dictionary for large-scale images. Empirically, we experiment with both space-invariant and space-varying sensing matrices and demonstrate the critical role of spatial diversity in measurements. Simulation results confirm that the presented algorithm outperforms the typical non-adaptive sparse recovery based on offline-learned universal dictionaries.
\end{abstract}

\begin{IEEEkeywords}
Blind compressed sensing, dictionary learning, sensor diversity, adaptive image recovery.
\end{IEEEkeywords}

\IEEEpeerreviewmaketitle

\section{Introduction}
\label{sec:intro}

The theory of Compressed Sensing (CS) establishes that the combinatorial problem of recovering the sparsest vector from a limited number of linear measurements can be solved in a polynomial time given that the measurements satisfy certain isometry conditions \cite{Candes_2008}. CS can be directly applied for recovering signals that are naturally sparse in the standard basis. Meanwhile, CS has been extended to work with many other types of natural signals that can be represented by a sparse vector using a \textit{dictionary} \cite{Candes_2011}. As an alternative to model-based dictionaries such as wavelets \cite{Donoho_1994}, Dictionary Learning (DL) \cite{V1} is a data-driven algorithmic approach to build sparse representations for natural signals. 

Learning dictionaries over large-scale databases of training images is a time and memory intensive process which results in a universal dictionary that works for most types of natural images. Meanwhile, several variations of DL algorithms have been proposed for real-time applications to make the sparse representations more \textit{adaptive}\footnote{We must emphasize the difference between adapting and learning dictionaries although the two terms are sometimes used interchangeably. In this paper, by dictionary learning we refer to its typical usage, i.e. the process of applying the DL algorithm to a large database of training images that produces a universal dictionary. Adapting a dictionary is the process of applying the DL algorithm to only one or a small number of possibly corrupted images to produce a dictionary that specifically performs well for those images.} in applications such as image denoising \cite{Elad_2006}, image inpainting \cite{Mairal_2008} and most recently, compressed sensing \cite{bcs1}. Particularly, the last application has been termed Blind Compressed Sensing (BCS) to differentiate it from the normal CS where the dictionary is assumed to be known and fixed. Clearly, one would expect BCS to improve CS recovery when the optimal sparsity basis is unknown, which is the case in most real-world applications. Unfortunately, the existing work on BCS for imaging is lacking in two directions: ($a$) empirical evaluations and ($b$) mathematical justification for the general case. These issues are discussed further below.  

\begin{itemize}

\item {\bf Empirical evaluations:} In existing BCS works, such as \cite{bcs2}, empirical evaluations on images are mainly limited to the image inpainting problem which can be viewed as a CS problem where the compressive measurements are in the standard basis. In \cite{bcs3}, the generic CS problem is only tested on artificially generated sparse vectors. Tested images and different running scenarios for the algorithms in \cite{bcs2, bcs3, GlobalSIP13, CDL_2012} are rather limited and arguably not adequate in indicating the strengths and weaknesses of BCS in real-world imaging applications. Finally and most importantly, existing studies fail to compare the adaptive BCS recovery with the non-adaptive CS recovery based on universally learned dictionaries. 

\item {\bf Mathematical justification:} The original BCS effort \cite{bcs1} identifies the general unconstrained BCS problem as ill-posed. Subsequently, various structural constraints were proposed for the learned dictionary and were shown to ensure uniqueness at the cost of decreased flexibility. In a following effort \cite{bcs2}, a different strategy was used to ensure the uniqueness without enforcing structural constraints on the dictionary. However, the uniqueness was only justified for the class of sparse signals that admit the block-sparse model by exploiting recent results from the area of low-rank matrix completion \cite{Recht_2009}. Finally, \cite{bcs3} and \cite{GlobalSIP13} take empirical approaches toward unconstrained DL based on compressive measurements but do not provide any justification for the uniqueness, convergence or the accuracy of the proposed DL algorithm.

\end{itemize}

The present work is different from existing efforts in BCS, both in terms of goals and methodology. In addition to the goal of having an objective function with a unique optimum, we would like the learned dictionary to be as close as possible to the ideal dictionary that is based on running the DL algorithm over the complete image. In other words, our goals include both convergence to a unique solution and high accuracy of the solution. Since no prior information is available about the structure of the ideal dictionary or the underlying signal, our method does not impose extra structural constraints on the learned dictionary\footnote{The constraint of having bounded column norm or Frobenius-norm of the dictionary, which is used in virtually every dictionary learning algorithm, does not constrain the dictionary structure other than bounding it or its columns inside the unit sphere. Some examples of structural constraints used in \cite{bcs1} and \cite{CDL_2012} respectively are block-diagonal dictionaries and sparse dictionaries.} or the sparse coefficients. 

Similar to most efforts in the area of compressive imaging, including the BCS framework, we employ a block compressed sensing or \textit{block-CS} scheme for measurement and recovery of images \cite{block-CS1, block-CS2}. Unlike dense-CS, where the image is recovered as a single vector using a dense sampling matrix, block-CS attempts to break down the high dimensional dense-CS problem into many small-sized CS problems for each non-overlapping block of the image. Some advantages of block-CS are: ($a$) block-CS results in a block-diagonal sampling matrix which significantly reduces the amount of required memory for storing large-scale sampling matrices, ($b$) decoding in extremely high dimensions\footnote{A typical consumer image has an order of $10^6$ pixels which would make it impractical to be recovered as a single vector using existing sparse recovery methods with cubic or quadratic time complexities.} is computationally challenging in dense-CS and ($c$) sparse modeling or learning sparsifying dictionaries for high-dimensional global image characteristics is challenging and not well studied. Specifically, we study a block-CS scheme where each block is sampled using a distinct sampling matrix and show that it is superior to using a fixed block sampling matrix for BCS. One of our goals in this paper is to outperform non-adaptive sparse image recovery using universal dictionaries based on well-known DL algorithms such as online-DL \cite{Mairal_2010} and K-SVD \cite{Elad_2006} while overcoming challenges such as overfitting. Rather than focusing on new DL algorithms for BCS, we focus on the relationship between the block-CS measurements and the BCS performance in an unconstrained setup.

This paper is organized as follows. In Section \ref{sec:review} we review the dictionary learning problem under the settings of complete and compressive data. Before describing the details of our algorithm in Section \ref{sec:algorithm}, we present our main contributions regarding the uniqueness conditions and the DL accuracy in the presence of partial data in Section \ref{sec:math}. Simulation results are presented in Section \ref{sec:simo}. Finally, we present the conclusion and a discussion of future directions in Section \ref{sec:discuss}.

\subsection{Notation}
\label{subsec:notation}
Throughout the paper, we use the following rules. Upper-case letters are used for matrices and lower-case letters are used for vectors and scalars. $I_n$ denotes the identity matrix of size $n$.  We reserve the following notation: $N$ is the total number of blocks in an image, $n$ is the size of each block (e.g. an $8\times 8$ block has size $64$), $m$ is the number of compressive measurements per block (usually $m\leq n$), $p$ is the number of atoms in a dictionary, $t$ is the iteration count, $D\in\mathbb{R}^{n\times p}$ denotes a dictionary, $x_j\in\mathbb{R}^n$ represents the vectorized image block (column-major) number $j$, $\alpha_j\in\mathbb{R}^p$ is the representation of $x_j$ (i.e. $x_j\approx D\alpha_j$), $\Phi_j\in\mathbb{R}^{m\times n}$ denotes the measurement matrix for block number $j$, $y_j\in\mathbb{R}^m$ denotes the vector of compressive measurements (i.e. $y_j=\Phi_j x_j$). For simplicity, we drop the block index subscripts in $\alpha_j$, $x_j$, $\Phi_j$ and $y_j$ when a single block is under consideration. Similarly, we omit the iteration superscript in $\alpha_j^{(t)}$ and $D^{(t)}$ when a single iteration is under study and it does not create confusion.

The vector $\ell_p$ norm is defined as $\|x\|_p = \left( \sum_i |x_i|^p \right)^{\frac{1}{p}}$. The matrix operators $A\otimes B$ and $A\odot B$ respectively represent the Kronecker and the Hadamard (or element-wise) products. The operator $\vect(A)$ reshapes a matrix $A$ to its column-major vectorized format. The matrix inner product is defined as $\langle A,B \rangle=\Tr(A^T B)$ with $\Tr(A)$ denoting the matrix trace, i.e. the sum of diagonal entries of $A$. The Frobenius-norm of $A$ is defined as $\|A\|_F = \left( \sum_{i j} |A_{i j}|^2 \right)^{\frac{1}{2}}$.

Finally, due to the frequent usage of Lasso regression \cite{Lasso_2004} in this paper, we use the following abstractions:
\begin{eqnarray*}
\mathcal{L}(x,D,\lambda,\alpha) &=& \frac{1}{2}\| x - D\alpha \|_2^2 + \lambda \| \alpha \|_1 \\
\mathcal{L}_{\min}(x,D,\lambda) &=& \min_{\alpha} \frac{1}{2}\| x - D\alpha \|_2^2 + \lambda \| \alpha \|_1 \\ 
\mathcal{L}_{\min}^{\arg}(x,D,\lambda) &=& \arg\min_{\alpha} \frac{1}{2}\| x - D\alpha \|_2^2 + \lambda \| \alpha \|_1
\label{eq:8293827947}
\end{eqnarray*}
In words, $\mathcal{L}_{\min}(x,D,\lambda)$ represents the model misfit and $\alpha^*=\mathcal{L}_{\min}^{\arg}(x,D,\lambda)$ denotes the sparse coefficient vector. Also, note the obvious relationship:
$$\mathcal{L}(x,D,\lambda,\mathcal{L}_{\min}^{\arg}(x,D,\lambda))=\mathcal{L}_{\min}(x,D,\lambda)$$

\section{Problem Statement and Prior Art}
\label{sec:review}

\subsection{An overview of the dictionary learning problem}
\label{subsec:overview}

The Dictionary Learning (DL) problem can be compactly expressed as:
\begin{equation}
D^* = \arg\min_{D} \psi(D)
\tag{P1}
\label{eq:DL}
\end{equation}
where $\psi(D)$ represents the collective model misfit\footnote{There are alternative ways of expressing the DL problem that are related. For example, authors in \cite{schnass_2010}, propose to minimize the sum of $\ell_1$ norms of coefficient vectors for a fixed (zero) representation error for each sample $x_j$.}:
\begin{equation*}
\psi(D) = \sum_{j=1}^{N} \mathcal{L}_{\min}(x_j,D,\lambda)
\end{equation*}

As noted in \cite{V1} for a similar formulation of DL, (\ref{eq:DL}) represents a \textit{bi-level} optimization problem :
\begin{itemize}
\item The inner layer (also known as the lower level) problem consists of solving $N$ Lasso problems to get $\psi(D)$.
\item The outer layer (or upper level) problem consists of finding a $D$ that minimizes $\psi(D)$. 
\end{itemize}

Note that even for large-scale images ($N\gg 1$) the lower level optimization can be handled efficiently by parallel programming because each block is processed independently. However, in a batch-DL algorithm\footnote{Batch processing refers to the processing of all $N$ blocks at once while online processing refers to the one-by-one processing of blocks in a streaming fashion.}, in contrast to the online-DL \cite{Mairal_2010}, the upper level problem is centralized and combines the information collected from all blocks. In this paper, we use the batch-DL approach to stay consistent with the mathematical analysis. However, in Section \ref{sec:algorithm}, an efficient algorithm is described for solving the batch problem. Similar to \cite{Mairal_2010}, the batch algorithm and its analysis can be extended to online-DL for the best efficiency. Hereafter, we omit the prefix `batch-' in batch-DL for simplicity. 

The typical DL strategy that is used in most works \cite{Elad_2006, Mairal_2008, Mairal_2010, Aharon_2008, ILS, RLS, MOD} is to iterate between the inner and outer optimization problems until convergence to a local minimum. Expressed formally, the iterative procedure is:
\begin{equation}
D^{(t+1)} = \arg\min_{D} \psi^{(t)}(D)
\label{eq:DL_iter}
\end{equation}
with
\begin{eqnarray*}
\psi^{(t)}(D) &=& \sum_{j=1}^{N} \mathcal{L}(x_j,D,\lambda,\alpha_j^{(t)}) \\
\alpha_j^{(t)} &=& \mathcal{L}_{\min}^{\arg}(x_j,D^{(t)},\lambda)
\end{eqnarray*}

The algorithm starts from an initial dictionary $D=D^{(0)}$ that can be selected to be, for example, the overcomplete discrete cosine frame \cite{Elad_2006}.

Perhaps surprisingly, the solution of (\ref{eq:DL}) is trivial without the additional constraint of having a bounded dictionary norm. To explain more, one can always reduce the model misfit $\mathcal{L}(x,D,\lambda,\alpha)$ by multiplying $D$ with a scalar $s>1$ and multiplying $\alpha$ with $1/s$, thus reducing the $\ell_1$ norm of $\alpha$ while keeping $x - D\alpha$ fixed, leading to $\|D^*\|_F\rightarrow \infty$ and $\|\alpha\|_2\rightarrow 0$. There are two typical bounding methods to solve this issue that are reviewed in e.g. \cite{Engan_2003, Davies_2008}: $a$) bounding the $\ell_2$ norm of each dictionary column or $b$) bounding $\|D^*\|_F$. In this work, we use the second approach, i.e. the bound $\|D\|_F\leq const.$, since it does not enforce a uniform distribution of column norms which makes the sparse representation more adaptive. As pointed out in \cite{Davies_2008}, using a Frobenius-norm bound results in a weighted sparse representation problem (at the inner level) where some coefficients can have more priority over others in taking non-zero values. Additionally, having bounded column norms is a stronger constraint which makes the analysis more difficult when the dictionary is treated in its vectorized format (this becomes more clear in Section \ref{sec:math}). The typical method for bounding the dictionary is by projecting back the updated dictionary (at the end of each iteration) inside the constraint set. More details are provided in Section \ref{sec:algorithm} where we describe the DL algorithm.

\subsection{Dictionary learning from compressive measurements}
\label{subsec:case_CI}

The problem of CS is to recover a sparse signal $x\in\mathbb{R}^n$, or a signal that can be approximated by a sparse vector, from a set of linear measurements $y=\Phi x\in\mathbb{R}^m$. When $m<n$, the linear system is under-determined and the solution set is infinite. However, it is not difficult to show that for a sufficiently sparse $x$ the solution to $y=\Phi x$ is unique \cite{from-sparse}. Unfortunately, the problem of searching for the sparsest $x$ subject to $y=\Phi x$ is NP-hard and impractical to solve for high-dimensional $x$. Meanwhile, the CS theory indicates that this problem can be solved in a polynomial time, using sparsity promoting solvers such as Lasso \cite{Lasso_2004}, given that $\Phi$ satisfies the Restricted Isometry Property (RIP) \cite{Candes_2008}. 

CS also applies to a dense $x$ when it has a sparse representation of the form $x=D\alpha$ (with a sparse $\alpha$). Measurements can be expressed as $y=P\alpha$, where $P=\Phi D$ is called the projection matrix. It has been shown that most random designs of $\Phi$ would yield RIP with high probabilities \cite{Baraniuk_2008}. The compressive imaging problem can be expressed as:
\begin{equation}
\hat{x}=D.\mathcal{L}_{\min}^{\arg}(y,\Phi D,\lambda)
\label{eq:313654666}
\end{equation}
The well-known \textit{basis pursuit} signal recovery corresponds to the following asymptotic solution \cite{Donoho_2008}:
\begin{equation}
\hat{x}=\lim_{\lambda\rightarrow 0^+} D.\mathcal{L}_{\min}^{\arg}(y,\Phi D,\lambda)
\label{eq:654654645}
\end{equation}

Hereafter, we focus on the block-CS framework where each $x_j$ represents an image block and $y_j=\Phi_j x_j$ represents the vector of compressive measurements for that block. The iterative DL procedure based on block-CS measurements can be written as:
\begin{equation}
D^{(t+1)} = \arg\min_{D} \hat{\psi}^{(t)}(D)
\label{eq:DL_CS}
\end{equation}
with
\begin{eqnarray*}
\hat{\psi}^{(t)}(D) &=& \sum_{j=1}^{N} \mathcal{L}(y_j,\Phi_j D,\lambda,\alpha_j^{(t)}) \\
\alpha_j^{(t)} &=& \mathcal{L}_{\min}^{\arg}(y_j,\Phi_j D^{(t)},\lambda)
\end{eqnarray*}

To distinguish (\ref{eq:DL_CS}) from the normal DL in (\ref{eq:DL_iter}) and other BCS formulations \cite{bcs1,bcs2,bcs3}, we refer to this problem as \textit{Dictionary Learning from linear Measurements} or simply DL-M. 

We could also arrange the block-wise measurements into a single system of linear equations:
\begin{equation}
\left[
\begin{array}{c}
y_1 \\
y_2 \\
\vdots \\
y_N	
\end{array}
\right] = 
\Phi \left(I_N \otimes D \right)
\left[
\begin{array}{c}
\alpha_1 \\
\alpha_2 \\
\vdots \\
\alpha_N	
\end{array}
\right]
\label{eq:24524552}
\end{equation}
where
\begin{equation}
\Phi = \left[
\begin{array}{c c c c}
\Phi_1 	&  				&  				&  				\\
				& \Phi_2 	&  				&  				\\
				& 				& \ddots 	&  				\\	
				& 				&  				& \Phi_N 
\end{array}
\right]
\label{eq:72982793847}
\end{equation}
represents the block-diagonal measurement matrix. Our results can be easily extended to dense-CS, i.e. CS with a dense $\Phi$. Although, the utility of dense-CS would not allow sequential processing of blocks as required by an online-DL framework, a batch-DL framework is compatible with dense-CS.

\section{Mathematical Analysis}
\label{sec:math}

The benefits of using a distinct $\Phi_j$ for each block can be understood intuitively \cite{GlobalSIP13}. However, it is important to study the asymptotic behavior of DL when $N\rightarrow \infty$, as well as the non-asymptotic bounds for a finite $N$. 

In the first part of this section, we prove that the iterative DL-M algorithm returns a unique solution with a probability that approaches one for large $N$. Specifically, we show that the outer problem, known as the `dictionary update' stage, 
$$D^{(t+1)}=\arg\min_D\hat{\psi}^{(t)}(D)$$ 
is unique for fixed $\alpha_j^{(t)}$'s and also every inner problem 
$$\forall j: \alpha_j^{(t)}=\mathcal{L}_{\min}^{\arg}(y_j,\Phi_j D^{(t)},\lambda)$$ 
is unique for a fixed $D^{(t)}$. Therefore, starting from an initial point $D=D^{(0)}$, the sequence of DL-M iterations forms a unique path. 

To specify the accuracy of the DL-M algorithm, we measure the expectation 
\begin{equation}
\mathbb{E}_{\Phi} \left\{ 
\| \arg\min_D\hat{\psi}^{(t)}(D) - \arg\min_D\psi^{(t)}(D) \|_F^2 
\right\}
\label{eq:deviation}
\end{equation}
starting from the same (fixed) $\alpha_j^{(t)}$'s. Meanwhile, the inner problem is precisely a noisy CS problem and its accuracy has been thoroughly studied \cite{Candes_2011}. Specifically, when $m=O(k_j\log p)$ where $k_j$ denotes the sparsity of $\alpha_j^{(t)}$, the inner CS problem for block $j$ can be solved exactly. The presented error analysis for a finite $N$ is limited to a single iteration of dictionary update. Nevertheless, the asymptotic conclusion as we present is that the DL-M and DL solutions converge as $N$ approaches infinity. 

Based on the above remarks, our analysis is focused on a single iteration of DL-M. Therefore, for simplicity, we drop the iteration superscript in the rest of this section unless required. 

First, we write $\hat{\psi}(D)$ in the standard quadratic format: 
\begin{eqnarray*}
\hat{\psi}(D) &=& 
\frac{1}{2} \sum_{j=1}^{N}\| y_j - \Phi_j D\alpha_j \|_2^2 + \lambda \sum_{j=1}^N \|\alpha_j \|_1 \\ &=&
\frac{1}{2}\sum_{j=1}^N y_j^T y_j + \lambda \sum_{j=1}^N \|\alpha_j \|_1 +\\ && 
\frac{1}{2}\sum_{j=1}^N \alpha_j^T D^T\Phi_j^T \Phi_j D \alpha_j - 
\sum_{j=1}^N y_j^T \Phi_j D \alpha_j
\label{eq:835726}
\end{eqnarray*}

We can further write:
\begin{eqnarray*}
\alpha_j^T D^T \Phi_j^T \Phi_j D \alpha_j &=&
\Tr(\alpha_j^T D^T \Phi_j^T \Phi_j D \alpha_j) \\ &=&
\Tr(D^T\Phi_j^T \Phi_j D \alpha_j \alpha_j^T) \\ &=&
\langle D, \Phi_j^T \Phi_j D \alpha_j \alpha_j^T \rangle \\ &=&
\vect (D)^T \vect(\Phi_j^T \Phi_j D \alpha_j \alpha_j^T) \\ &=&
\vect (D)^T (\alpha_j \alpha_j^T \otimes \Phi_j^T \Phi_j) \vect( D )
\label{eq:4545645}
\end{eqnarray*}
and
\begin{eqnarray*}
y_j^T \Phi_j D \alpha_j &=&
\Tr(y_j^T \Phi_j D \alpha_j) \\ &=&
\Tr(\alpha_j y_j^T \Phi_j D) \\ &=&
\langle \Phi_j^T y_j \alpha_j^T, D \rangle \\ &=& 
\vect(\Phi_j^T y_j \alpha_j^T)^T \vect (D)
\label{eq:294758}
\end{eqnarray*}

Letting $\bm{d} = \vect(D)$ and $g(\bm{d}) = \hat{\psi}(D)$, the standard quadratic form of $\hat{\psi}(D)$ can be written as:
\begin{equation}
g(\bm{d}) = \frac{1}{2} \bm{d}^T Q \bm{d} + f^T \bm{d} + c
\label{eq:242342}
\end{equation}
with
\begin{eqnarray*}
Q &=& \sum_{j=1}^N \alpha_j \alpha_j^T \otimes \Phi_j^T \Phi_j \\
f &=& -\sum_{j=1}^N \Phi_j^T y_j \alpha_j^T \\
c &=& \frac{1}{2}\sum_{j=1}^N y_j^T y_j + \lambda \sum_{j=1}^N \|\alpha_j \|_1
\label{eq:2873684276}
\end{eqnarray*}

Next, we shall specify the stochastic construction of block-CS measurements that we term the \textit{BIG measurement} scheme. 

\begin{definition}
{\bf (\textit{BIG measurement})} In a Block-based Independent Gaussian or BIG measurement scheme, each entry $(k,l)$ of each block measurement matrix $\left[ \Phi_j \right]_{k,l}$ is independently drawn from a zero-mean random Gaussian distribution with variance $1/m$.
\end{definition}

The $1/m$ variance guarantees that $\mathbb{E}\{\Phi_j^T\Phi_j\}=I_n$. Note that although our analysis focuses on Gaussian measurements, it is straightforward to extend it to the larger class of sub-Gaussian measurements which includes the Rademacher and the general class of (centered) bounded random variables \cite{Baraniuk_2008}. 

\subsection{Uniqueness}
Before presenting the uniqueness results, we review the matrix extension of the Chernoff inequality \cite{Tropp_2012} that is summarized in the following lemma. 

\begin{lemma}
(Matrix Chernoff, {Theorem 5.1.1 in \normalfont \cite{Tropp_2012}}). 
Consider a finite sequence $\{X_k\}\subset\mathbb{R}^{n\times n}$ of independent, random and positive semidefinite Hermitian matrices that satisfy $\lambda_{\max}(X_k)\leq R$. Define the random matrix $Y=\sum_k X_k$. 
Compute the expectation parameters: 
$\mu_{\max}=\lambda_{\max}(\mathbb{E}Y)$ and 
$\mu_{\min}=\lambda_{\min}(\mathbb{E}Y)$
Then, for $\theta>0$,
\begin{equation}
\mathbb{E} \lambda_{\max}(Y) \leq \frac{e^{\theta}-1}{\theta}\mu_{\max}+\frac{1}{\theta}R\log n
\label{eq:34534534}
\end{equation}
and
\begin{equation}
\mathbb{E} \lambda_{\min}(Y) \geq \frac{1-e^{-\theta}}{\theta}\mu_{\min}-\frac{1}{\theta}R\log n
\label{eq:98927834}
\end{equation}
Furthermore,
\begin{equation}
\mathbb{P}\{\lambda_{\max}(Y)\geq (1+\delta)\mu_{\max}\}\leq n 
\left( \frac{e^{\delta}}{(1+\delta)^{(1+\delta)}} \right)^{\mu_{\max}/R}
\label{eq:9029384}
\end{equation}
for $\delta\geq 0$ and
\begin{equation}
\mathbb{P}\{\lambda_{\min}(Y)\leq (1-\delta)\mu_{\min}\}\leq n 
\left( \frac{e^{-\delta}}{(1-\delta)^{(1-\delta)}} \right)^{\mu_{\min}/R}
\label{eq:234392879}
\end{equation}
\label{lemma:matrix_chernoff}
\end{lemma}

This lemma will be used to show that, with a high probability, the Hessian matrix $Q$ of $g(\bm{d})$, which is a sum of random independent matrices, is full rank and invertible. 

In the following theorem, let $\mu_0$ denote the lower bound of the smallest eigenvalue of the covariance matrix $\mathbb{E}\{\alpha_j\alpha_j^T\}$. Note that the covariance matrix must be full rank, or equivalently $\mu_0>0$, otherwise even the original dictionary learning problem (based on the complete data) would not result in a unique solution. On top of that, the magnitude of $\mu_0$ has a direct impact on the condition number of the Hessian matrix and the numerical stability of DL-M, as well as DL. 

\begin{theorem}
\label{theorem:uniqueness} In a BIG measurement scheme with $N\gg\frac{n\log n}{\mu_0}$, $g(\bm{d})$ has a unique minimum with a high probability.
\end{theorem}

\begin{proof}
Taking the derivative of $g(\bm{d})$ with respect to $\bm{d}$ and letting it equal to zero results in the linear equation $Q\bm{d}=f$. Thus, to prove that the solution is unique, we must show that the Hessian matrix is invertible (with a high probability) for large $N$. Equivalently, we must show that the probability $\mathbb{P}\{\lambda_{\min}(Q)>0\}$ is close to one when $N\gg \frac{n\log n}{\mu_0}$. Since $Q$ is a sum of independent matrices, we may use the matrix Chernoff inequality from Lemma \ref{lemma:matrix_chernoff}. Hence, we must compute the following quantities:
$$R = \sup\lambda_{\max}(\alpha_j \alpha_j^T \otimes \Phi_j^T \Phi_j)$$ and 
\begin{eqnarray*}
\mu_{\min} &=& \lambda_{\min}(\mathbb{E}\sum_j \alpha_j \alpha_j^T \otimes \Phi_j^T \Phi_j) \\ &=& 
\lambda_{\min}(\sum_j \alpha_j \alpha_j^T \otimes I_n)
\end{eqnarray*}

Using properties of the Kronecker product,
\begin{eqnarray*}
\lambda_{\max}(\alpha_j \alpha_j^T \otimes \Phi_j^T \Phi_j) &=&
\lambda_{\max}(\alpha_j \alpha_j^T)
\lambda_{\max}(\Phi_j^T \Phi_j) \\ &\leq&
\|\alpha_j\|_2^2 (1+\delta)
\end{eqnarray*}
where the inequality holds with probability 
$1-2e^{-m(\delta^2/4-\delta^3/6)}$ 
for a random Gaussian measurement matrix $\Phi_j$ \cite{Baraniuk_2008}. Suppose the energy of every $\alpha_j$ is bounded by some constant $v=O(n)$ given that $x_j$ has bounded energy. Therefore, with probability $1-2e^{-m(\delta^2/4-\delta^3/6)}$, $R\leq(1+\delta)v$. Roughly speaking, assuming that pixel intensities are in the range $[0,1]$, $R\approx n$. 

Again, using properties of the Kronecker product,
\begin{eqnarray*}
\mu_{\min} &=&
\lambda_{\min}(\sum_{j=1}^N \alpha_j \alpha_j^T) \\ &\approx&
N\lambda_{\min}(\mathbb{E}\{\alpha_j \alpha_j^T\}) \geq
N\mu_0
\end{eqnarray*}

Based on Lemma \ref{lemma:matrix_chernoff}, specifically using (\ref{eq:234392879}) with $\delta=1$,

\begin{equation}
\mathbb{P}\{\lambda_{\min}(Q)\leq 0\}\leq n e^{-\mu_{\min}/R}
\leq n e^{-N\mu_0/R}
\label{eq:28378423748}
\end{equation}

Requiring $n e^{-N\mu_0/R} \ll 1$, and that $R\approx n$, is equivalent to $N\gg \frac{n\log n}{\mu_0}$.
\end{proof}

We have established that the upper level problem results in a unique solution with a high probability\footnote{Clearly, projecting the resultant dictionary onto the space of matrices with a constant Frobenius-norm would preserve the uniqueness since there is only a single point on the sphere of constant-norm matrices that is closest to the current dictionary.}. To complete this subsection, we use the following result from \cite{lasso-unique} that implies the lower level problem $\mathcal{L}_{\min}^{\arg}(y_j,\Phi_j D,\lambda)$ is unique. 

\begin{lemma}
{\normalfont \cite{lasso-unique}} If entries of $D$ are drawn from a continuous probability distribution on $\mathbb{R}^{n\times p}$, then for any $x_j$ and $\lambda >0$ the lasso solution $\mathcal{L}_{\min}^{\arg}(x_j,D,\lambda)$ is unique with probability one.
\label{lemma:lasso}
\end{lemma}

Since each $\Phi_j$ is drawn a continuous probability distribution in the BIG measurement scheme, $\Phi_j D$ is also distributed continuously in the space $\mathbb{R}^{m\times p}$ and this establishes that $\mathcal{L}_{\min}^{\arg}(y_j,\Phi_j D,\lambda)$ is unique. 

\subsection{Accuracy}
In this subsection, we shall compute stochastic upper bounds for the $\ell_2$ distance between the DL-M solution and the DL solution for a single iteration of dictionary update and for fixed $\alpha_j$'s. The extension of these results to multiple iterations is left as a future work. As before, in the following results, we omit the iteration superscript for simplicity. 

Let us define the corresponding standard quadratic form $\bar{g}(\bm{d})=\psi(D)$ for the upper-level DL problem:
\begin{equation}
\bar{g}(\bm{d}) = \frac{1}{2} \bm{d}^T \bar{Q} \bm{d} + \bar{f}^T \bm{d} + \bar{c}
\label{eq:3562356}
\end{equation}
where
\begin{eqnarray*}
\bar{Q} &=& \sum_{j=1}^N \alpha_j \alpha_j^T \otimes I_n \\
\bar{f} &=& -\sum_{j=1}^N x_j \alpha_j^T \\
\bar{c} &=& \frac{1}{2}\sum_{j=1}^N x_j^T x_j + \lambda \sum_{j=1}^N \|\alpha_j \|_1
\label{eq:45356346}
\end{eqnarray*}

For BIG measurements, $\mathbb{E}\{\Phi_j^T\Phi_j\}=I_n$. Therefore, it is easy to verify that $\bar{Q}=\mathbb{E}\{Q\}$, $\bar{f}=\mathbb{E}\{f\}$ and $\bar{c}=\mathbb{E}\{c\}$ where it is assumed that the data and coefficients are fixed. This leads to the following lemma that points out the unbiasedness of the compressive objective function.

\begin{lemma}
\label{lemma:unbiased_objective}
$g(\bm{d})$ is an unbiased estimator of $\bar{g}(\bm{d})$ for BIG measurements:
$$\mathbb{E}\{g(\bm{d})\}=\bar{g}(\bm{d})$$
\end{lemma}

\begin{proof}
\begin{eqnarray*}
\mathbb{E}\{g(\bm{d})\} &=& 
\mathbb{E}\{ \frac{1}{2} \bm{d}^T Q \bm{d} + f^T \bm{d} + c \} \\ &=&
\frac{1}{2} \bm{d}^T \mathbb{E}\{ Q\} \bm{d} + \mathbb{E}\{f\} ^T \bm{d} + \mathbb{E}\{c\} \\ &=&
\frac{1}{2} \bm{d}^T \bar{Q} \bm{d} + \bar{f} ^T \bm{d} + \bar{c} \\ &=&
\bar{g}(\bm{d})
\end{eqnarray*}
\end{proof}

The following crucial lemma implies that the two objective functions $g(\bm{d})$ and $\bar{g}(\bm{d})$ get arbitrarily close as the number of blocks ($N$) approaches infinity.

\begin{lemma}
\label{lemma:block_cs}
{\normalfont {\bf (}\cite{block-cs}, {\bf based on Theorem III.1)}} In a BIG scheme,
\begin{eqnarray*}
\mathbb{P}\left\{
\frac{\left|\sum_{j=1}^N \|\Phi_j(x_j- D\alpha_j)\|_2^2 - 
\sum_{j=1}^N \|x_j-D\alpha_j\|_2^2\right| }
{\sum_{j=1}^N \|x_j-D\alpha_j\|_2^2} > \epsilon \right\} \\
\leq 2 e^{-C\epsilon m^2 \gamma}
\end{eqnarray*}
where $\gamma\in [1,N]$ can be computed as:
$$\gamma = \frac{\sum_{j=1}^N \|x_j-D\alpha_j\|_2^2}{\max_{j} \|x_j-D\alpha_j\|_2^2}$$
\end{lemma}

We have simplified and customized Theorem III.1 of \cite{block-cs} for our problem here. Specifically, the bounds in \cite{block-cs} are tighter but more difficult to interpret. The above lemma states that, with a constant (high) probability that depends on $m$, the deviation $\epsilon$ is inversely proportional to $N$ when the signal energy is evenly distributed among blocks. 

Lemma \ref{lemma:block_cs} can be further customized by noticing that 
$$\sum_{j=1}^N \|\Phi_j(x_j- D\alpha_j)\|_2^2 - \sum_{j=1}^N \|x_j-D\alpha_j\|_2^2 = 
2\left[ g(\bm{d})-\bar{g}(\bm{d}) \right]$$ 
and that
$\bar{g}(\bm{d})\geq \frac{1}{2}\sum_{j=1}^N \|x_j-D\alpha_j\|_2^2$, leading to
\begin{equation}
\mathbb{P}\{|g(\bm{d})-\bar{g}(\bm{d})|>\epsilon \bar{g}(\bm{d}) \} \leq 2 e^{-C\epsilon m^2 \gamma}
\label{eq:283748237}
\end{equation}

Hereafter, to simplify the notation, let 
\begin{equation}
\hat{\bm{d}}=\arg\min_{\bm{d}}g(\bm{d})
\label{eq:quadratic_DLM}
\end{equation}
and
\begin{equation}
\bm{d}^*=\arg\min_{\bm{d}}\bar{g}(\bm{d})
\label{eq:quadratic_DL}
\end{equation}.

The following theorem provides upper bounds for the expectation $\mathbb{E}\{\|\hat{\bm{d}}-\bm{d}^*\|_2^2\}$. Suppose that, for a fixed $N$, there exists a positive constant $\mu_1$ such that $\lambda_{\min}(Q)\geq \mu_1$. Clearly, this is a reasonable assumption for large $N$ according to Theorem \ref{theorem:uniqueness}. More specifically, according to Lemma \ref{lemma:matrix_chernoff}, $\mathbb{E}\{\lambda_{\min}(Q)\}$ grows linearly with $N$ (because $\mu_{\min}$ grows linearly with $N$).

\begin{theorem}
\label{theorem:accuracy}
$\hat{\bm{d}}$ and $\bm{d}^*$ converge as $N$ approaches infinity. Specifically,
$$\mathbb{E}\{\|\hat{\bm{d}}-\bm{d}^*\|_2^2\}\leq \frac{2\epsilon}{\mu_1}\bar{g}(\bm{d}^*)$$
\end{theorem}

\begin{proof}
We start by writing the Taylor expansion of the quadratic function $g(\bm{d})$ at $\bm{d}^*$:
$$g(\bm{d}^*) = g(\hat{\bm{d}}) + 
\left. \frac{\partial g(\bm{d})}{\partial \bm{d}}^T\right\vert_{\bm{d}=\hat{\bm{d}}}(\bm{d}^*-\hat{\bm{d}}) + 
\frac{1}{2}(\bm{d}^*-\hat{\bm{d}})^T Q (\bm{d}^*-\hat{\bm{d}})$$

Since $\hat{\bm{d}}=\arg\min_{\bm{d}}g(\bm{d})$,
$$\left. \frac{\partial g(\bm{d})}{\partial \bm{d}}^T\right\vert_{\bm{d}=\hat{\bm{d}}} = 0$$
and we can write
\begin{eqnarray*}
g(\bm{d}^*) - g(\hat{\bm{d}}) &=&
\frac{1}{2}(\bm{d}^*-\hat{\bm{d}})^T Q (\bm{d}^*-\hat{\bm{d}}) \\ &\geq& 
\frac{\lambda_{\min}(Q)}{2}\|\bm{d}^*-\hat{\bm{d}}\|_2^2 \\ &\geq&
\frac{\mu_1}{2}\|\bm{d}^*-\hat{\bm{d}}\|_2^2
\label{eq:2839837}
\end{eqnarray*}

Taking the expected value of both sides
\begin{equation}
\mathbb{E}\{\|\bm{d}^*-\hat{\bm{d}}\|_2^2\} \leq \frac{2}{\mu_1} \mathbb{E}\{g(\bm{d}^*) - g(\hat{\bm{d}})\}
\label{eq:2398274}
\end{equation}

From Lemma \ref{lemma:block_cs} we know that with a probability of at least $1-2e^{-C\epsilon m^2 \gamma}$ the following inequalities hold
\begin{equation}
(1-\epsilon) \bar{g}(\hat{\bm{d}}) \leq g(\hat{\bm{d}})\leq (1+\epsilon) \bar{g}(\hat{\bm{d}})
\label{eq:3453453}
\end{equation}
\begin{equation}
(1-\epsilon) \bar{g}(\bm{d}^*) \leq g(\bm{d}^*)\leq (1+\epsilon) \bar{g}(\bm{d}^*)
\label{eq:232434234}
\end{equation}

On the other hand, we have the following inequalities at the optimum points of $g(\bm{d})$ and $\bar{g}(\bm{d})$
\begin{equation}
\bar{g}(\bm{d}^*) \leq \bar{g}(\hat{\bm{d}})
\label{eq:54764476}
\end{equation}
and
\begin{equation}
g(\hat{\bm{d}}) \leq g(\bm{d}^*)
\label{eq:234525}
\end{equation}

It is easy to check that, by combining (\ref{eq:3453453}), (\ref{eq:232434234}), (\ref{eq:54764476}) and (\ref{eq:234525}), we can arrive at the following inequality:
$$(1-\epsilon) \bar{g}(\bm{d}^*) \leq g(\hat{\bm{d}})\leq (1+\epsilon) \bar{g}(\bm{d}^*)$$
or equivalently,
\begin{equation}
-\epsilon \bar{g}(\bm{d}^*) \leq \bar{g}(\bm{d}^*) - g(\hat{\bm{d}}) \leq \epsilon \bar{g}(\bm{d}^*)
\label{eq:3456363}
\end{equation}

Taking the expected value, we get
\begin{equation}
-\epsilon \bar{g}(\bm{d}^*) \leq \mathbb{E}\{ \bar{g}(\bm{d}^*) - g(\hat{\bm{d}}) \} \leq \epsilon \bar{g}(\bm{d}^*)
\label{eq:45642664}
\end{equation}

From Lemma \ref{lemma:unbiased_objective} we know that
$$\mathbb{E}\{g(\bm{d}^*)\}=\bar{g}(\bm{d}^*)=\mathbb{E}\{\bar{g}(\bm{d}^*)\}$$

Therefore,
\begin{equation}
0\leq \mathbb{E}\{ g(\bm{d}^*) - g(\hat{\bm{d}}) \} = \mathbb{E}\{ \bar{g}(\bm{d}^*) - g(\hat{\bm{d}}) \}
\label{eq:1231213}
\end{equation}

Use (\ref{eq:45642664}) and (\ref{eq:1231213}) to arrive at
\begin{equation}
\mathbb{E}\{g(\bm{d}^*) - g(\hat{\bm{d}})\} \leq \epsilon \bar{g}(\bm{d}^*)
\label{eq:238243648}
\end{equation}
which, along with (\ref{eq:2398274}), completes the proof.
\end{proof}

Note that $\epsilon$ is inversely proportional to $N$ and $\bar{g}(\bm{d}^*)$ grows linearly with $N$. Furthermore, using (\ref{eq:234392879}) for a constant probability, it can be shown that $\mu_1$ increases linearly with $N$, making the ratio $\frac{2\epsilon}{\mu_1}\bar{g}(\bm{d}^*)$ arbitrarily small for $N\rightarrow \infty$. 

Finally, we show that after projection onto $\|D\|_F=c$, where $c$ is a positive constant, the upper bound of the estimation error would still approach zero for large $N$. By noticing that $\|D\|_F=\|\bm{d}\|_2$, this projection can be written as
\begin{equation*}
\hat{\bm{d}} \leftarrow \frac{c\hat{\bm{d}}}{\|\hat{\bm{d}}\|_2},\;\;
\bm{d}^* \leftarrow \frac{c\bm{d}^*}{\|\bm{d}^*\|_2}
\end{equation*} 

It is easy to show that
\begin{equation*}
\left\| 
\frac{c\hat{\bm{d}}}{\|\hat{\bm{d}}\|_2} - 
\frac{c\bm{d}^*}{\|\bm{d}^*\|_2} \right\|_2 \leq
c \max \left\{ \frac{1}{\|\hat{\bm{d}}\|_2},
\frac{1}{\|\bm{d}^*\|_2} \right\} 
\|\hat{\bm{d}}-\bm{d}^*\|_2
\label{eq:3928734827}
\end{equation*}

Using $\hat{\bm{d}}=Q^{-1}f$ and $\bm{d}^*=\bar{Q}^{-1}\bar{f}$, lower bounds for $\|\hat{\bm{d}}\|_2$ and $\|\bm{d}^*\|_2$ can be computed as
$$\|\hat{\bm{d}}\|_2\geq \frac{\|f\|_2}{\lambda_{\max}(Q)},\;\;
\|\bm{d}^*\|_2\geq \frac{\|\bar{f}\|_2}{\lambda_{\max}(\bar{Q})}$$

Therefore
$$\max \left\{ \frac{1}{\|\hat{\bm{d}}\|_2},
\frac{1}{\|\bm{d}^*\|_2} \right\} \leq 
\max \left\{ \frac{\lambda_{\max}(Q)}{\|f\|_2},
\frac{\lambda_{\max}(\bar{Q})}{\|\bar{f}\|_2} \right\}$$

Using Lemma \ref{lemma:matrix_chernoff} and other well-established concentration inequalities, one can find stochastic upper bounds for quantities above. However, given that $\lambda_{\max}(Q)$, $\lambda_{\max}(\bar{Q})$, $\|\bar{f}\|_2$ and $\|f\|_2$ scale linearly with $N$, we can safely conclude that the estimation error remains bounded by an arbitrarily small number as $N$ approaches infinity. Moreover, intuitively speaking, the $\ell_2$ norm of $\hat{\bm{d}}$ and $\bm{d}^*$ tends to increase before projection (which is the reason for bounding the dictionary in the first place) and the ratios $c/\|\hat{\bm{d}}\|_2$ and $c/\|\bm{d}^*\|_2$ are likely to be smaller than one, resulting in a decrease in the estimation error.

\section{The Main Algorithm}
\label{sec:algorithm}

The employed algorithm for DL-M is based on ({\ref{eq:DL_iter}). However, as explained below, we introduce several modifications to decrease the computational complexity and speed up the convergence. Similar to other DL algorithms, such as \cite{Elad_2006}, the proposed algorithm consists of two stages that are called the sparse coding stage and the dictionary update stage. We describe them individually in the following subsections.
 
\subsection{The sparse coding stage}

As we mentioned in Section \ref{subsec:case_CI}, the basis pursuit (exact) CS recovery is the limit of the Lasso solution $\mathcal{L}_{\min}^{\arg}(y,\Phi D,\lambda)$ as $\lambda$ approaches zero \cite{Donoho_2008}}. However, a truly sparse and exact representation is usually not possible using any dictionary with a finite size. As a result, in sparse recovery of natural images, $\lambda$ is usually selected to be a small number rather than zero even in noiseless scenarios \cite{Mairal_2010, Mairal_2009}. Our algorithm starts from a coarse and overly sparse representation, by selecting the initial $\lambda$ to be large, and gradually reduces $\lambda$ until the desired balance between the total error sum of squares and the sparsity is achieved. The idea behind this modification is that the initial dictionary is suboptimal and not capable of giving an exact sparse representation. However, as the iterations pass, the dictionary becomes closer to the optimal dictionary and $\lambda$ must be decreased to get a sparse representation that closely adheres to the measurements.

Initializing the counter at $t=0$ and starting from an initial dictionary $D=D^{(0)}$, the sparse coding stage consists of performing the following optimization:
\begin{equation}
\forall j: \alpha_j^{(t)}=\arg\min_{\alpha} 
\frac{1}{2}\| y_j - \Phi_j D^{(t)}\alpha \|_2^2 + \lambda^{(t)} \| \alpha \|_1
\label{eq:lower_Phi}
\end{equation}

We deploy an exponential decay for $\lambda^{(t)}$: 
\begin{equation}
\lambda^{(t)} = \max\{ \lambda_0 e^{- \frac{t}{T_*} .\log\left(\frac{\lambda_0}{\lambda_*} \right) }, \lambda_* \}
\label{eq:546564}
\end{equation}

According to (\ref{eq:546564}), $\lambda$ is decreased from $\lambda=\lambda_0$ to $\lambda=\lambda_*$ in $t=T_*$ iterations and stays fixed at $\lambda=\lambda_*$ henceforth. For an exact recovery, $\lambda_*=0$ is seemingly a plausible choice. However, for the reasons mentioned earlier, we set $\lambda_*$ to a very small but non-zero value that is specified in the simulations section.

\subsection{The dictionary update stage}

The quadratic optimization problem of (\ref{eq:quadratic_DLM}) can be computationally inefficient to solve. More specifically, solving (\ref{eq:quadratic_DLM}) in one step requires computing the inverse of $Q$ (if it exists) which has roughly a time complexity of $O(n^3 p^3)$ and is a memory intensive operation. The strategy that we employ in this paper, similar to what was proposed in \cite{V1,Engan_2003,Aharon_2008}, is to perform a gradient descent step: 
\begin{equation}
D^{(t+1)}=D^{(t)}-\mu^{(t)} \nabla_D \hat{\psi}^{(t)}(D)
\label{eq:Gstep}
\end{equation}
where $\nabla_D \psi^{(t)}(D)$ can be computed efficiently:
\begin{equation}
\nabla_D \hat{\psi}^{(t)}(D)= -\sum_{j=1}^N 
\Phi_j^T (y_j - \Phi_j D^{(t)}\alpha_j^{(t)}){\alpha_j^{(t)}}^T
\label{eq:2374827634}
\end{equation}

The step size $\mu^{(t)}>0$ can be iteratively decreased with $t$ \cite{Aharon_2008} or it can be optimized in a \textit{steepest descent} fashion that is described below\footnote{If $Q$ is well-conditioned, a single step of steepest descent can give a close approximation of the the solution of (\ref{eq:quadratic_DLM})}. 

The optimal value of the step size $\mu^{(t)}_*$ can be computed using a simple line search \cite{Deb_1995}. However, for a quadratic objective function, we can derive $\mu^{(t)}_*$ in a closed form as shown below. Let $G^{(t)}=G^{(t)}(D)=\nabla_D \hat{\psi}^{(t)}(D)$. Then,
\begin{eqnarray*}
\mu^{(t)}_* &=& \arg\min_{\mu}
\sum_{j=1}^{N} \|y_j- \Phi_j D^{(t+1)} \alpha_j^{(t)} \|_2^2 \\ &=& \arg\min_{\mu}
\sum_{j=1}^{N} \|y_j- \Phi_j (D^{(t)}+\mu G^{(t)}) \alpha_j^{(t)} \|_2^2 
\label{eq:242563365}
\end{eqnarray*}

Writing the optimality conditions for the objective function above, we can arrive at the following solution for $\mu^{(t)}_*$:

\begin{equation}
\mu^{(t)}_* = 
\frac{\| G^{(t)} \|_F^2}{\sum_{j=1}^N \|\Phi_j G^{(t)}\alpha_j^{(t)} \|_2^2}
\label{eq:29829384729}
\end{equation}

Since the initial dictionary consists of $p$ unit-norm columns, $\|D^{(0)}\|_F=\sqrt{p}$. As discussed in Section \ref{sec:review}, DL results in an unbounded dictionary if no constraint is put on the dictionary norm or the norm of its columns. Here, we employ a Frobenius bound on $D$ because it lets different dictionary columns have distinct norms. Specifically, after each update $D^{(t+1)}=D^{(t)}-\mu^{(t)}_* G^{(t)}$, we ensure the constraint $\|D^{(t+1)}\|_F = \sqrt{p}$. This is done by multiplying $D^{(t+1)}$ with $p^{\frac{1}{2}}\|D^{(t+1)}\|_F^{-1}$. Algorithm \ref{alg1} gives a summary of these steps.

\begin{algorithm}[htb]
\caption{Dictionary learning from compressive samples.}
\begin{algorithmic}
\REQUIRE $D$, $y_j$ and $\Phi_j$ for every $j$, $\lambda_0$, $\lambda_*$, $T_*$, $T_{\max}$
\STATE Initialization $t \leftarrow 1$, $D^{(1)} \leftarrow D$
\WHILE{$t < T_{\max}$}
\STATE \hrulefill $\;$ Sparse Coding \hrulefill
\STATE Compute $\alpha_j^{(t)}$ for every $j$:
$$\alpha_j^{(t)}=\arg\min_{\alpha} \frac{1}{2}\lVert y_j - \Phi_j D^{(t)}\alpha \rVert_2^2 
+ \lambda^{(t)} \lVert \alpha \rVert_1$$ \hspace{\algorithmicindent} with 
$$\lambda^{(t)} = \max\{ \lambda_0 e^{- \frac{t}{T_*} .\log\left(\frac{\lambda_0}{\lambda_*} \right) }, \lambda_* \}$$
\STATE \hrulefill $\;$ Dictionary Update \hrulefill
\STATE Compute $D^{(t+1)}=D^{(t)}+\mu^{(t)}_* G^{(t)}$ using:
$$G^{(t)} = -\sum_{j=1}^N \Phi_j^T (y_j-\Phi_j^T D^{(t)} \alpha_j^{(t)}){\alpha_j^{(t)}}^T$$ 
\hspace{\algorithmicindent} and 
$$\mu^{(t)}_* = \frac{\| G^{(t)} \|_F^2}{\sum_{j=1}^N \|\Phi_j G^{(t)}\alpha_j^{(t)} \|_2^2}$$
\STATE Normalize the dictionary:
$$D^{(t+1)}\leftarrow \frac{p^{\frac{1}{2}} D^{(t+1)}}{\|D^{(t+1)}\|_F}$$
\STATE \hrulefill
\STATE $t \leftarrow t+1$
\ENDWHILE
\RETURN $D^{(T_{\max})}$
\end{algorithmic}
\label{alg1}
\end{algorithm}

We conclude this section by demonstrating the performance of the proposed DL-M algorithm using space-varying block measurements through an empirical test and compare it with the scenario if space-invariant (or fixed) block measurements were employed. 

\subsection{Testing the algorithm: a real-world example}
\label{subsec:two_experiment}

In this subsection, we test the performance of the described DL-M algorithm for block-CS using $a$) space-varying and $b$) fixed sampling matrices\footnote{Fixed sampling matrices have been used in dictionary learning from block compressive measurements, such as in \cite{Wu_2012}}. In this experiment, non-overlapping $8\times 8$ patches from the Barbara's image (shown in Figure \ref{fig:testImages}) are sampled at 50\% sampling ratio, i.e. $n=64$ and $m=32$. Pixel intensities are normalized to be in the range $[0,1]$ and blocks are vectorized and centered. We have used a fixed $\lambda=0.01$ for this experiment. Specifically, we are interested in the real-time reconstruction Peak-Signal-to-Noise-Ratio or PSNR. For the varying case, sampling matrices are generated randomly for each block according to independent and identically distributed Gaussian distribution with variance $1/m$ (equivalent to the BIG measurement scheme described before). Furthermore, we orthogonalize and normalize rows of each generated sampling matrix so that measurements are weighted equally in the learning problem. In the fixed sampling matrix case, a single random matrix is generated and employed for all blocks. The initial dictionary for this experiment is a redundant dictionary of size $p=256$ that has been trained using the method of K-SVD and distributed in \cite{Elad_2006}. This dictionary, which is shown in Figure \ref{fig:TwoDic} (top), serves as the benchmark among redundant dictionaries for natural images. 

Figure \ref{fig:example_plot} shows the real-time PSNR graphs, i.e. PSNR measured after each iteration of the DL-M algorithm, associated with both cases of space-varying and fixed sampling matrices as well as the non-adaptive PSNR. Several crucial observations can be made about Figure \ref{fig:example_plot}. Although a slight improvement is achieved after the first few iterations of the DL-M algorithm based on the fixed sampling matrix, the PSNR is decreased subsequently. The decline in the PSNR performance is expected because the Hessian matrix $Q$ from (\ref{eq:242342}), would be low-rank when $\Phi_1=\Phi_2=\dots=\Phi_N$, making DL-M an ill-posed problem. 

\begin{figure}[htb]
\centering
\includegraphics[trim=0in 0in 0in 0in,
clip=true,width=1\linewidth]{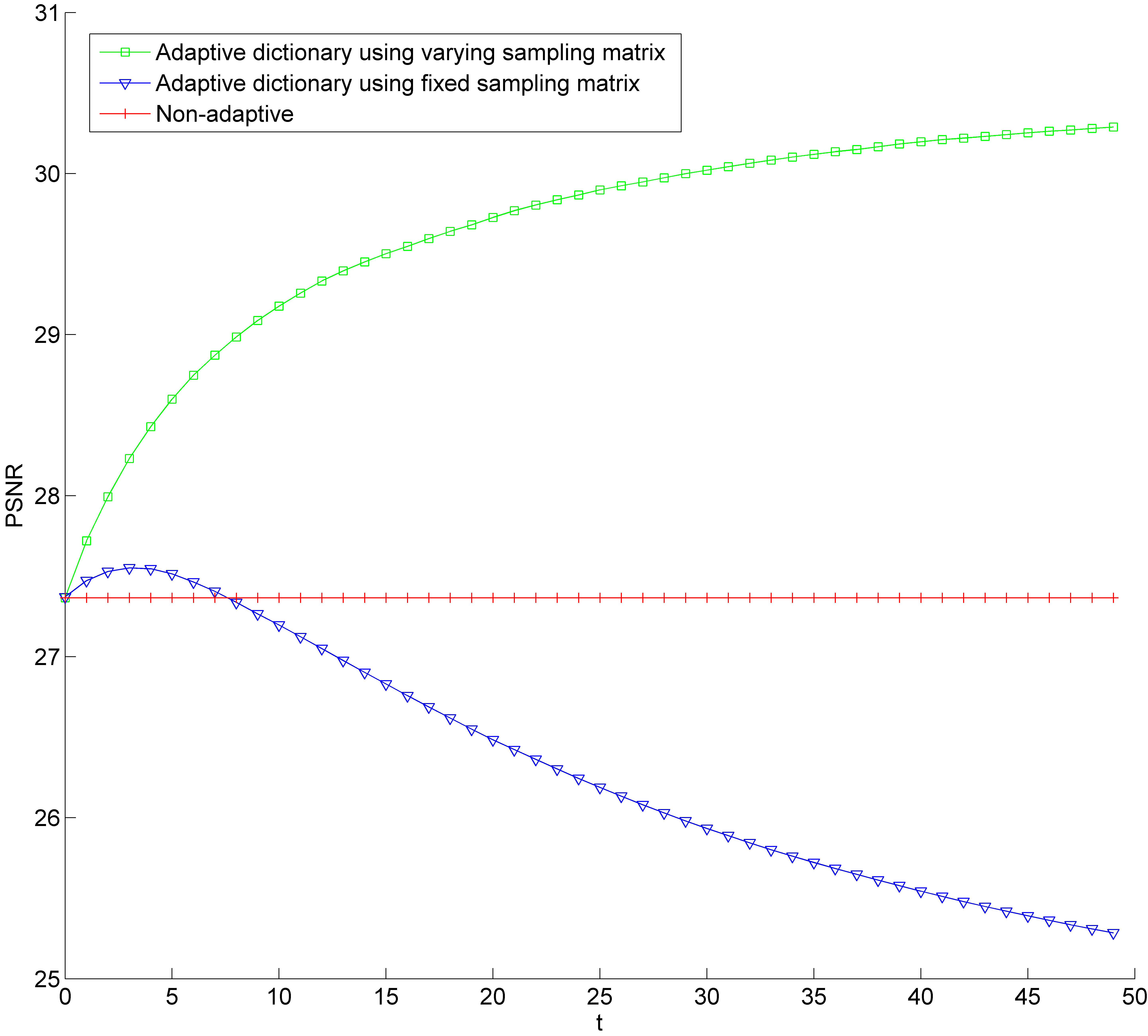}
\caption{Graphs of the real-time PSNR for cases of adaptive recovery using space-varying and fixed sampling matrices versus the non-adaptive recovery using a universal dictionary. The horizontal axis shows the iteration count $t$.}
\label{fig:example_plot}
\end{figure}

Meanwhile, the PSNR graph for the space-varying block-CS shows significant enhancement with respect to the initial PSNR just after a few iterations. It is helpful to visually inspect the dictionaries before and after adaptation. These dictionaries are shown in Figure \ref{fig:TwoDic}. The resultant adaptive dictionary is shown at the bottom of Figure \ref{fig:TwoDic}. A careful inspection reveals that some of the texture from the input image have been captured within the adaptive dictionary. These texture patterns are the parts of the image that could not be compactly represented using the initial dictionary. The misfit of the initial dictionary results in artifacts in the non-adaptive recovery around the textured areas as can be seen in Figure \ref{fig:Barbi_sub1}. The recovery after adaptation is shown in Figure \ref{fig:Barbi_sub2} for comparison. Note the improvements in the textured areas in the adaptively recovered image compared to the initial recovery using the universal dictionary. 

\begin{figure}[htb]
\centering
\includegraphics[trim=0in 0in 0in 0in,
clip=true,width=1\linewidth]{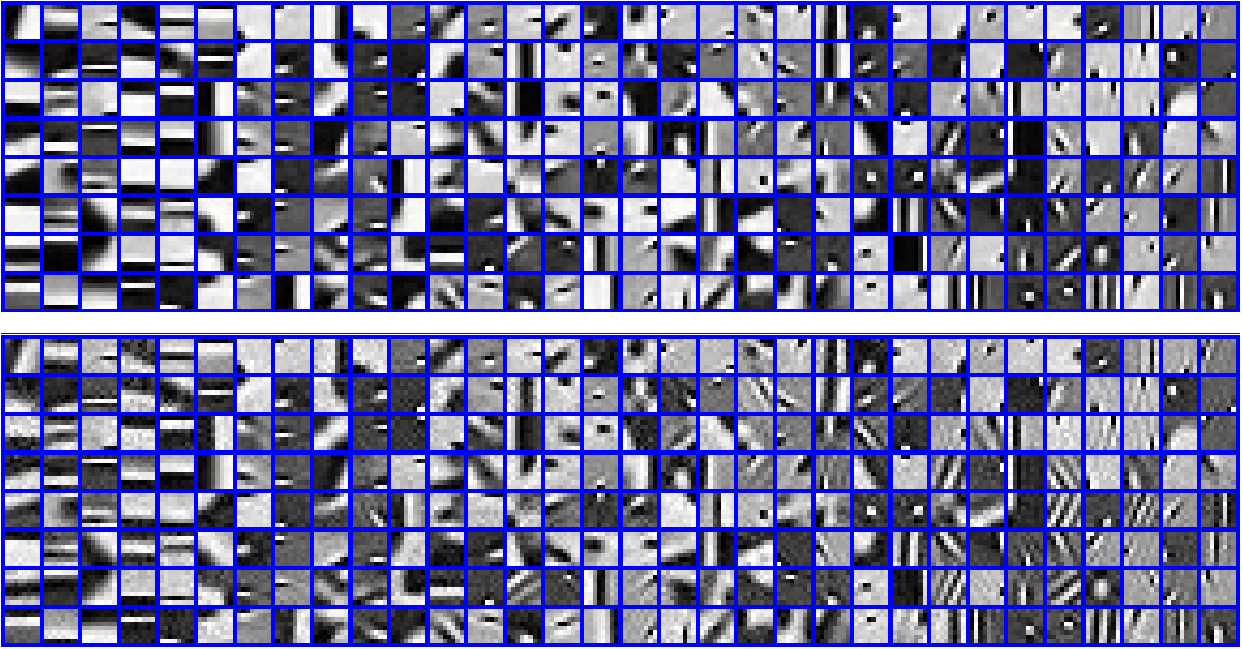}
\caption{Dictionaries. Top: the initial universal dictionary, bottom: the adapted dictionary based on 50\% sampling from Barbara's image.}
\label{fig:TwoDic}
\end{figure}

\begin{figure*}[htb]
\centering
	\subfloat[Using the universal dictionary\label{fig:Barbi_sub1}]{%
		\includegraphics[width=.45\textwidth]{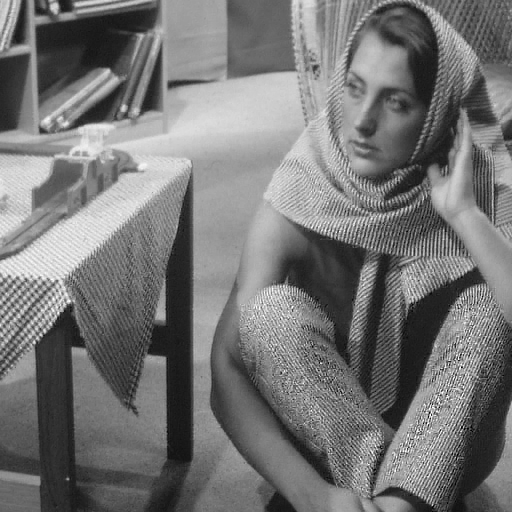}
	}
	\subfloat[Using the adaptive dictionary\label{fig:Barbi_sub2}]{%
		\includegraphics[width=.45\textwidth]{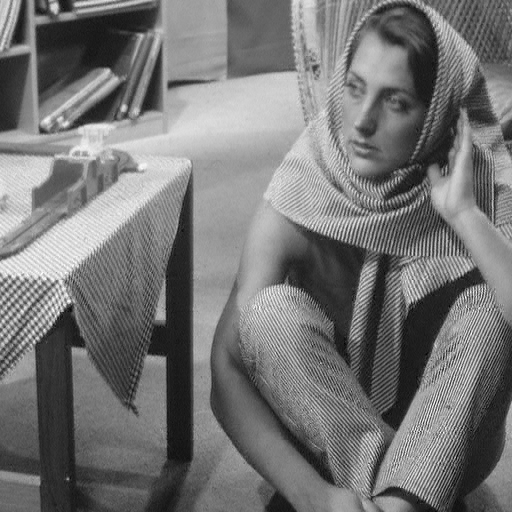}
	}
	\caption{Recoveries from 50\% sampling using the fixed universal and adaptive dictionaries.}
	\label{fig:Barbi}
\end{figure*}

\section{Simulations and Results}
\label{sec:simo}

\begin{table*}
\centering
\begin{tabular}{|c|c|c|c|c|c|c|c|c|c|c|c|c|c|c|}
\cline{2-15}
\multicolumn{1}{c|}{} & \multicolumn{7}{c|}{Universal dictionary} & 
\multicolumn{7}{c|}{Adaptive dictionary} \\
\cline{2-15}
\multicolumn{1}{c|}{} & 5\% & 10\% & 20\% & 25\% & 30\% & 40\% & 50\% & 5\% & 10\% & 20\% & 25\% & 30\% & 40\% & 50\% \\
\hline
Barbara & 
22.14 & 23.01 & 25.19 & 26.52 & 27.43 & 29.19 & 31.09 & 
22.20 & 23.22 & 25.99 & 27.88 & 29.39 & 32.05 & 34.85 \\ 
\hline
boat & 
23.47 & 24.84 & 28.47 & 30.80 & 32.41 & 35.24 & 38.18 & 
23.51 & 25.03 & 29.11 & 31.45 & 33.03 & 35.95 & 39.23 \\ 
\hline
bridge & 
21.98 & 23.15 & 26.05 & 27.88 & 29.16 & 31.67 & 34.79 & 
22.03 & 23.34 & 26.50 & 28.24 & 29.48 & 31.97 & 35.13 \\ 
\hline
couple & 
23.59 & 24.94 & 28.48 & 30.80 & 32.41 & 35.43 & 38.77 & 
23.65 & 25.33 & 29.29 & 31.45 & 32.95 & 35.87 & 39.29 \\ 
\hline
fingerprint 1 & 
17.77 & 19.31 & 23.66 & 26.88 & 29.27 & 33.87 & 39.02 & 
17.95 & 20.71 & 26.18 & 29.29 & 31.65 & 36.03 & 41.29 \\ 
\hline
fingerprint 2 & 
17.58 & 19.13 & 22.89 & 25.03 & 26.33 & 28.41 & 30.46 & 
17.84 & 21.35 & 27.59 & 30.74 & 32.95 & 35.81 & 38.85 \\ 
\hline
Flintstones & 
17.12 & 18.59 & 22.72 & 25.52 & 27.41 & 30.61 & 33.45 & 
17.13 & 18.73 & 23.48 & 26.29 & 28.17 & 31.22 & 33.97 \\ 
\hline
grass & 
12.51 & 13.29 & 15.36 & 16.85 & 17.99 & 20.32 & 23.01 & 
12.55 & 13.52 & 15.94 & 17.58 & 18.88 & 21.43 & 24.62 \\ 
\hline
hill & 
25.83 & 27.21 & 30.85 & 33.17 & 34.70 & 37.44 & 40.28 & 
25.89 & 27.76 & 31.74 & 33.83 & 35.26 & 37.91 & 41.05 \\ 
\hline
house & 
24.02 & 25.39 & 29.22 & 31.73 & 33.42 & 36.42 & 39.44 & 
24.09 & 25.89 & 32.27 & 35.96 & 38.27 & 41.83 & 45.51 \\ 
\hline
Lena & 
25.12 & 26.75 & 31.44 & 34.40 & 36.33 & 39.69 & 43.19 & 
25.20 & 27.56 & 32.39 & 34.98 & 36.76 & 39.93 & 43.35 \\ 
\hline
man & 
24.40 & 25.79 & 29.43 & 31.76 & 33.32 & 36.19 & 39.44 & 
24.46 & 26.11 & 30.14 & 32.28 & 33.76 & 36.52 & 39.70 \\ 
\hline
matches & 
21.16 & 22.90 & 27.83 & 30.82 & 32.62 & 35.73 & 39.35 & 
21.21 & 23.93 & 29.18 & 31.70 & 33.50 & 36.67 & 40.47 \\ 
\hline
shuttle & 
26.26 & 28.06 & 33.74 & 37.83 & 40.58 & 45.22 & 49.71 & 
26.39 & 29.33 & 35.70 & 39.38 & 41.91 & 46.20 & 50.53 \\ 

\hline
\end{tabular}
\caption{Recovery PSNRs for {\bf SET1}. PSNRs are in {\normalfont dB}. Percentage values are sampling ratios.}
\label{tab:set1}
\end{table*}

\begin{table*}
\centering
\begin{tabular}{|c|c|c|c|c|c|c|c|c|c|c|c|c|c|c|}
\cline{2-15}
\multicolumn{1}{c|}{} & \multicolumn{7}{c|}{Universal dictionary} & 
\multicolumn{7}{c|}{Adaptive dictionary} \\
\cline{2-15}
\multicolumn{1}{c|}{} & 5\% & 10\% & 20\% & 25\% & 30\% & 40\% & 50\% & 5\% & 10\% & 20\% & 25\% & 30\% & 40\% & 50\% \\
\hline
Barbara & 
21.17 & 21.94 & 23.77 & 24.88 & 25.65 & 27.09 & 28.55 & 
21.23 & 22.06 & 24.33 & 25.75 & 26.84 & 28.80 & 30.83 \\ 
\hline
boat & 
22.34 & 23.45 & 26.05 & 27.54 & 28.52 & 30.15 & 31.61 & 
22.37 & 23.52 & 26.34 & 27.81 & 28.76 & 30.49 & 32.32 \\ 
\hline
bridge & 
20.40 & 21.20 & 22.91 & 23.86 & 24.48 & 25.54 & 26.55 & 
20.44 & 21.22 & 23.01 & 23.93 & 24.52 & 25.68 & 27.02 \\ 
\hline
couple & 
22.39 & 23.44 & 25.97 & 27.46 & 28.45 & 30.12 & 31.76 & 
22.44 & 23.62 & 26.40 & 27.78 & 28.69 & 30.34 & 32.15 \\ 
\hline
fingerprint 1 & 
17.09 & 18.46 & 21.93 & 24.22 & 25.74 & 28.40 & 31.13 & 
17.21 & 19.43 & 23.51 & 25.53 & 26.87 & 29.13 & 31.62 \\ 
\hline
fingerprint 2 & 
17.09 & 18.51 & 21.83 & 23.64 & 24.75 & 26.48 & 28.13 & 
17.30 & 20.38 & 25.61 & 28.15 & 29.95 & 32.33 & 34.97 \\ 
\hline
Flintstones & 
16.53 & 17.83 & 21.14 & 23.17 & 24.43 & 26.40 & 27.96 & 
16.55 & 17.89 & 21.61 & 23.58 & 24.79 & 26.70 & 28.38 \\ 
\hline
grass & 
11.75 & 12.38 & 13.96 & 15.01 & 15.77 & 17.15 & 18.49 & 
11.78 & 12.50 & 14.28 & 15.39 & 16.21 & 17.68 & 19.33 \\ 
\hline
hill & 
24.19 & 25.20 & 27.47 & 28.74 & 29.56 & 30.94 & 32.30 & 
24.25 & 25.41 & 27.83 & 28.96 & 29.69 & 31.09 & 32.65 \\ 
\hline
house & 
23.14 & 24.34 & 27.44 & 29.34 & 30.57 & 32.62 & 34.60 & 
23.21 & 24.68 & 29.33 & 31.61 & 32.96 & 34.89 & 36.72 \\ 
\hline
Lena & 
24.06 & 25.41 & 28.77 & 30.66 & 31.80 & 33.69 & 35.45 & 
24.13 & 25.96 & 29.36 & 31.01 & 32.04 & 33.83 & 35.60 \\ 
\hline
man & 
23.10 & 24.17 & 26.65 & 28.05 & 28.93 & 30.46 & 31.97 & 
23.15 & 24.33 & 26.96 & 28.26 & 29.08 & 30.60 & 32.24 \\ 
\hline
matches & 
20.22 & 21.65 & 25.10 & 26.83 & 27.78 & 29.32 & 30.88 & 
20.26 & 22.31 & 25.82 & 27.23 & 28.10 & 29.56 & 31.14 \\ 
\hline
shuttle & 
25.47 & 27.14 & 31.98 & 35.11 & 37.16 & 40.49 & 43.69 & 
25.57 & 28.30 & 33.52 & 36.25 & 38.10 & 41.14 & 44.16 \\ 

\hline
\end{tabular}
\caption{Recovery PSNRs for {\bf SET2}. PSNRs are in {\normalfont dB}. Percentage values are sampling ratios.}
\label{tab:set2}
\end{table*}

\begin{table*}
\centering
\begin{tabular}{|c|c|c|c|c|c|c|c|c|c|c|c|c|c|c|}
\cline{2-15}
\multicolumn{1}{c|}{} & \multicolumn{7}{c|}{Universal dictionary} & 
\multicolumn{7}{c|}{Adaptive dictionary} \\
\cline{2-15}
\multicolumn{1}{c|}{} & 5\% & 10\% & 20\% & 25\% & 30\% & 40\% & 50\% & 5\% & 10\% & 20\% & 25\% & 30\% & 40\% & 50\% \\
\hline
Barbara & 
21.08 & 21.76 & 23.25 & 24.06 & 24.64 & 25.76 & 27.16 & 
21.17 & 21.95 & 23.70 & 24.82 & 25.78 & 27.60 & 29.76 \\ 
\hline
boat & 
22.24 & 23.29 & 25.83 & 27.33 & 28.31 & 30.04 & 31.69 & 
22.30 & 23.46 & 26.20 & 27.71 & 28.67 & 30.47 & 32.32 \\ 
\hline
bridge & 
20.36 & 21.17 & 22.97 & 23.96 & 24.62 & 25.80 & 27.06 & 
20.41 & 21.27 & 23.10 & 24.02 & 24.62 & 25.82 & 27.24 \\ 
\hline
couple & 
22.34 & 23.40 & 25.96 & 27.52 & 28.56 & 30.47 & 32.44 & 
22.41 & 23.67 & 26.39 & 27.79 & 28.72 & 30.51 & 32.46 \\ 
\hline
fingerprint 1 & 
16.98 & 18.22 & 21.40 & 23.49 & 24.95 & 27.62 & 30.53 & 
17.14 & 19.30 & 23.26 & 25.30 & 26.68 & 28.97 & 31.58 \\ 
\hline
fingerprint 2 & 
17.01 & 18.50 & 22.22 & 24.30 & 25.57 & 27.80 & 30.30 & 
17.30 & 20.55 & 25.95 & 28.70 & 30.75 & 33.50 & 36.67 \\ 
\hline
Flintstones & 
16.43 & 17.62 & 20.71 & 22.60 & 23.88 & 26.05 & 28.00 & 
16.47 & 17.78 & 21.33 & 23.29 & 24.58 & 26.62 & 28.51 \\ 
\hline
grass & 
11.71 & 12.27 & 13.65 & 14.63 & 15.37 & 16.84 & 18.50 & 
11.74 & 12.37 & 14.00 & 15.10 & 15.94 & 17.47 & 19.24 \\ 
\hline
hill & 
24.12 & 25.15 & 27.49 & 28.80 & 29.65 & 31.16 & 32.78 & 
24.20 & 25.42 & 27.84 & 29.04 & 29.80 & 31.27 & 32.92 \\ 
\hline
house & 
23.07 & 24.19 & 27.09 & 28.96 & 30.28 & 32.62 & 34.97 & 
23.13 & 24.51 & 29.08 & 31.43 & 32.80 & 34.83 & 36.75 \\ 
\hline
Lena & 
23.96 & 25.24 & 28.50 & 30.27 & 31.37 & 33.28 & 35.19 & 
24.06 & 25.84 & 29.17 & 30.78 & 31.81 & 33.58 & 35.44 \\ 
\hline
man & 
23.04 & 24.10 & 26.63 & 28.04 & 28.98 & 30.59 & 32.28 & 
23.10 & 24.32 & 26.97 & 28.30 & 29.14 & 30.73 & 32.51 \\ 
\hline
matches & 
20.13 & 21.59 & 25.08 & 26.73 & 27.68 & 29.28 & 31.03 & 
20.20 & 22.15 & 25.79 & 27.19 & 28.07 & 29.61 & 31.34 \\ 
\hline
shuttle & 
25.40 & 27.08 & 31.88 & 34.80 & 36.63 & 39.96 & 43.28 & 
25.56 & 28.27 & 33.39 & 36.06 & 37.89 & 40.97 & 44.13 \\ 

\hline
\end{tabular}
\caption{Recovery PSNRs for {\bf SET3}. PSNRs are in {\normalfont dB}. Percentage values are sampling ratios.}
\label{tab:set3}
\end{table*}

\begin{table*}
\centering
\begin{tabular}{|c|c|c|c|c|c|c|c|c|c|c|c|c|c|c|}
\cline{2-15}
\multicolumn{1}{c|}{} & \multicolumn{7}{c|}{Universal dictionary} & 
\multicolumn{7}{c|}{Adaptive dictionary} \\
\cline{2-15}
\multicolumn{1}{c|}{} & 5\% & 10\% & 20\% & 25\% & 30\% & 40\% & 50\% & 5\% & 10\% & 20\% & 25\% & 30\% & 40\% & 50\% \\
\hline
Barbara & 
20.60 & 21.07 & 22.50 & 23.53 & 24.33 & 25.97 & 28.06 & 
20.97 & 21.91 & 23.50 & 24.72 & 25.81 & 27.57 & 29.66 \\ 
\hline
boat & 
21.50 & 22.01 & 23.57 & 24.68 & 25.50 & 27.20 & 29.26 & 
21.95 & 23.15 & 24.91 & 26.11 & 27.10 & 28.63 & 30.48 \\ 
\hline
bridge & 
19.69 & 20.04 & 21.19 & 22.01 & 22.62 & 23.85 & 25.31 & 
20.07 & 20.91 & 22.22 & 23.06 & 23.74 & 24.81 & 26.12 \\ 
\hline
couple & 
21.61 & 22.13 & 23.71 & 24.86 & 25.72 & 27.44 & 29.41 & 
22.06 & 23.33 & 25.19 & 26.41 & 27.39 & 28.84 & 30.57 \\ 
\hline
fingerprint 1 & 
16.14 & 16.69 & 18.54 & 20.01 & 21.20 & 23.66 & 26.56 & 
16.87 & 18.35 & 20.95 & 22.75 & 24.22 & 26.14 & 28.41 \\ 
\hline
fingerprint 2 & 
16.16 & 16.99 & 19.35 & 20.93 & 22.08 & 24.26 & 26.71 & 
17.10 & 18.97 & 21.74 & 23.57 & 25.12 & 27.16 & 29.53 \\ 
\hline
Flintstones & 
15.60 & 16.07 & 17.66 & 18.88 & 19.88 & 21.96 & 24.45 & 
15.98 & 17.02 & 18.93 & 20.32 & 21.54 & 23.40 & 25.58 \\ 
\hline
grass & 
11.12 & 11.29 & 12.08 & 12.76 & 13.34 & 14.63 & 16.36 & 
11.30 & 11.84 & 13.09 & 13.99 & 14.80 & 16.06 & 17.66 \\ 
\hline
hill & 
23.43 & 23.96 & 25.50 & 26.52 & 27.26 & 28.73 & 30.46 & 
23.94 & 25.25 & 26.94 & 27.97 & 28.81 & 30.04 & 31.55 \\ 
\hline
house & 
22.29 & 22.77 & 24.36 & 25.61 & 26.61 & 28.76 & 31.36 & 
22.69 & 24.02 & 26.49 & 28.26 & 29.77 & 31.76 & 33.90 \\ 
\hline
Lena & 
23.17 & 23.83 & 25.71 & 27.00 & 27.95 & 29.89 & 32.11 & 
23.74 & 25.24 & 27.26 & 28.62 & 29.71 & 31.30 & 33.20 \\ 
\hline
man & 
22.28 & 22.75 & 24.23 & 25.26 & 26.04 & 27.57 & 29.45 & 
22.75 & 23.97 & 25.65 & 26.71 & 27.60 & 28.88 & 30.50 \\ 
\hline
matches & 
19.34 & 20.22 & 22.63 & 24.02 & 24.96 & 26.68 & 28.62 & 
19.92 & 21.74 & 23.96 & 25.23 & 26.22 & 27.66 & 29.39 \\ 
\hline
shuttle & 
24.58 & 25.55 & 28.27 & 30.20 & 31.64 & 34.50 & 37.99 & 
25.27 & 27.44 & 30.48 & 32.53 & 34.24 & 36.78 & 39.82 \\ 

\hline
\end{tabular}
\caption{Recovery PSNRs for {\bf SET4}. PSNRs are in {\normalfont dB}. Percentage values are sampling ratios.}
\label{tab:set4}
\end{table*}

\begin{table*}
\centering
\begin{tabular}{|c|c|c|c|c|c|c|c|c|c|c|c|c|c|c|}
\cline{2-15}
\multicolumn{1}{c|}{} & \multicolumn{7}{c|}{Universal dictionary} & 
\multicolumn{7}{c|}{Adaptive dictionary} \\
\cline{2-15}
\multicolumn{1}{c|}{} & 5\% & 10\% & 20\% & 25\% & 30\% & 40\% & 50\% & 5\% & 10\% & 20\% & 25\% & 30\% & 40\% & 50\% \\
\hline
Barbara & 
21.19 & 21.89 & 23.78 & 24.90 & 25.66 & 27.06 & 28.40 & 
21.31 & 21.90 & 24.35 & 25.81 & 26.86 & 28.68 & 30.53 \\ 
\hline
boat & 
22.38 & 23.43 & 25.98 & 27.36 & 28.25 & 29.78 & 31.24 & 
22.47 & 23.40 & 26.22 & 27.60 & 28.49 & 30.09 & 31.84 \\ 
\hline
bridge & 
20.52 & 21.23 & 22.89 & 23.79 & 24.38 & 25.40 & 26.41 & 
20.65 & 21.32 & 22.99 & 23.89 & 24.48 & 25.60 & 26.92 \\ 
\hline
couple & 
22.49 & 23.48 & 25.94 & 27.32 & 28.23 & 29.82 & 31.44 & 
22.56 & 23.50 & 26.32 & 27.65 & 28.50 & 30.06 & 31.79 \\ 
\hline
fingerprint 1 & 
17.15 & 18.45 & 22.08 & 24.17 & 25.52 & 27.94 & 30.51 & 
17.24 & 18.74 & 23.51 & 25.33 & 26.56 & 28.67 & 31.04 \\ 
\hline
fingerprint 2 & 
17.13 & 18.40 & 21.52 & 23.10 & 24.11 & 25.85 & 27.71 & 
17.21 & 18.93 & 24.74 & 26.79 & 28.29 & 30.52 & 33.15 \\ 
\hline
Flintstones & 
16.58 & 17.80 & 21.05 & 22.84 & 23.98 & 25.83 & 27.40 & 
16.62 & 17.77 & 21.38 & 23.12 & 24.21 & 26.04 & 27.77 \\ 
\hline
grass & 
11.84 & 12.43 & 14.08 & 15.12 & 15.84 & 17.13 & 18.35 & 
11.99 & 12.59 & 14.44 & 15.53 & 16.29 & 17.67 & 19.25 \\ 
\hline
hill & 
24.25 & 25.19 & 27.43 & 28.62 & 29.39 & 30.72 & 32.07 & 
24.33 & 25.20 & 27.77 & 28.88 & 29.59 & 30.94 & 32.49 \\ 
\hline
house & 
23.22 & 24.33 & 27.38 & 29.11 & 30.23 & 32.16 & 34.11 & 
23.33 & 24.47 & 28.70 & 30.81 & 32.19 & 34.20 & 36.16 \\ 
\hline
Lena & 
24.06 & 25.33 & 28.62 & 30.33 & 31.39 & 33.18 & 34.96 & 
24.14 & 25.40 & 29.08 & 30.65 & 31.62 & 33.32 & 35.11 \\ 
\hline
man & 
23.19 & 24.18 & 26.59 & 27.85 & 28.68 & 30.12 & 31.61 & 
23.28 & 24.19 & 26.83 & 28.03 & 28.80 & 30.24 & 31.87 \\ 
\hline
matches & 
20.23 & 21.54 & 24.92 & 26.53 & 27.46 & 29.04 & 30.66 & 
20.27 & 21.54 & 25.59 & 26.94 & 27.82 & 29.31 & 30.94 \\ 
\hline
shuttle & 
25.53 & 27.02 & 31.63 & 34.32 & 36.04 & 39.03 & 42.09 & 
25.60 & 27.37 & 32.87 & 35.27 & 36.88 & 39.55 & 42.48 \\ 

\hline
\end{tabular}
\caption{Recovery PSNRs for {\bf SET5}. PSNRs are in {\normalfont dB}. Percentage values are sampling ratios.}
\label{tab:set5}
\end{table*}

\begin{figure}[htb]
\centering
\includegraphics[trim=0in 0in 0in 0in,
clip=true,width=1\linewidth]{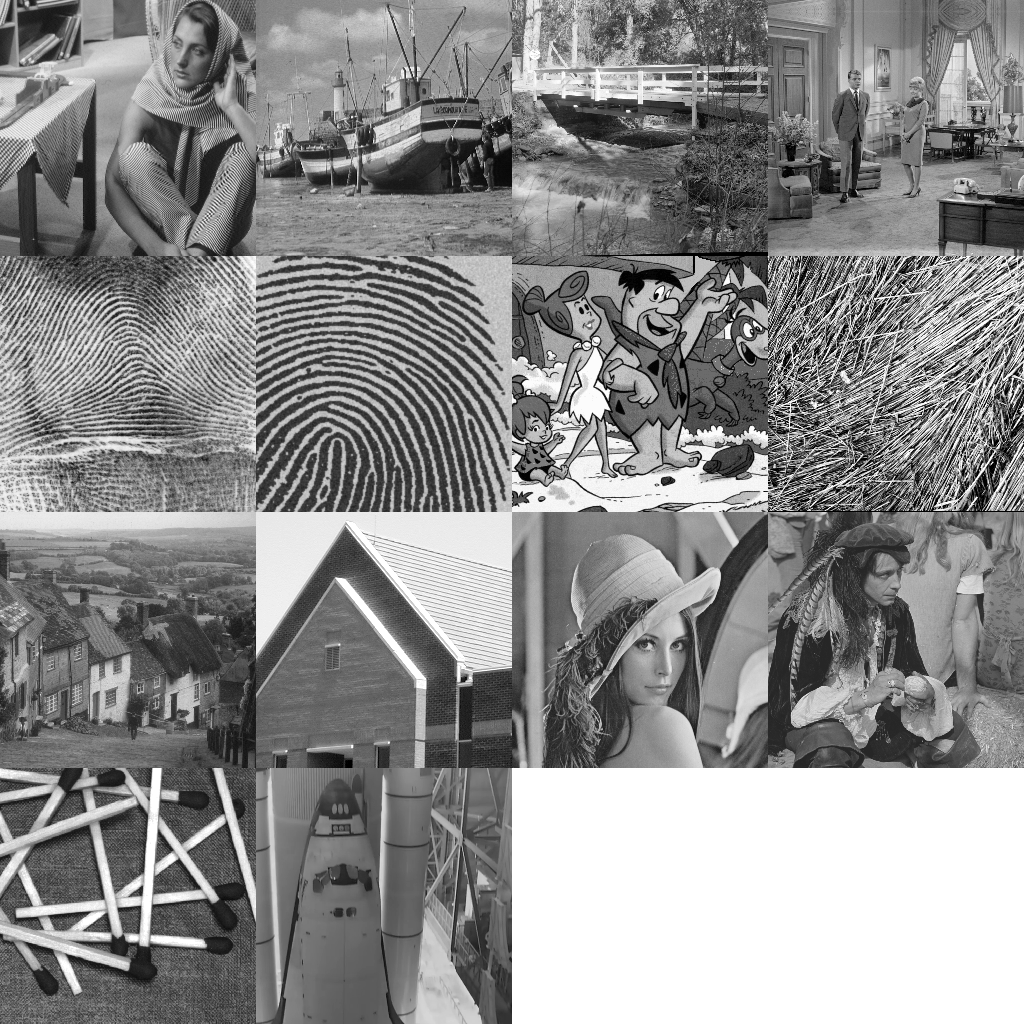}
\caption{Test images. From left to right and top to bottom: `Barbara', `boat', `bridge', `couple', `fingerprint 1', `fingerprint 2', `Flintstones', `grass', `hill', `house', `Lena', `man', `matches' and `shuttle'.}
\label{fig:testImages}
\end{figure}

\subsection{Simulation settings}
The set of $14$ test images, which are down-scaled due to the limited space, are shown in Figure \ref{fig:testImages}. Each test image has a resolution of $512\times 512$. In all simulations, we have used a fixed block size of $8\times 8$. Input parameters to Algorithm \ref{alg1} are $\lambda_0=0.05$, $\lambda_*=0.001$ and $T_{\max}=20$. However, we selected $T_*$ in the range $[1, 10]$ and proportional to $m$, the number of compressive measurements per block. 

This section includes five different simulation settings that are described in the following. 

{\bf SET1:} In this setting, we denoise\footnote{Here, denoising refers to the process of removing small representation coefficients, making the exact image representation truly sparse.} the test images prior to measurement by performing a Lasso regression for each (non-overlapping) block with $\lambda=0.05$ with respect to an ``ideal'' dictionary that is only known to the tester. This pre-processing guarantees that, although unknown to the learner, there exists an exact sparse representation for every test image. The ideal dictionary consists of $p=256$ atoms (a redundancy of factor $4$) that is learned based on the set of all overlapping blocks of the test image using the dictionary learning algorithm in \cite{Mairal_2010}. The initial dictionary is computed by cross-validation. Specifically, $10^6$ training patches were randomly selected from other images in Figure \ref{fig:testImages} for training the initial dictionary with $p=256$; this results in a distinct initial dictionary for each test image. For learning all ideal and initial dictionaries, we have set the maximum number of DL iterations at $1000$ with other learning parameters set at their default values \cite{Mairal_2009}. For example, the mini-batch size for the online learning algorithm is $400$ patches and $\lambda_1=0.15$.  

{\bf SET2:} This setting is identical to {\bf SET1}, except that we do not perform pre-processing on the test images. Therefore, these simulations are closer to the real case where there is no guarantee test images are sparse under any finite dictionary. 

{\bf SET3:} In this setting, for the initial dictionary, we have used a redundant dictionary that was trained over a large set of training images using the K-SVD method \cite{Elad_2006}. Testing with the K-SVD dictionary helps in benchmarking the performance of the proposed method. 

{\bf SET4:} In addition to testing the proposed DL-M algorithm with the state-of-the-art learned dictionaries, in this setting we test the algorithm with an ordinary orthogonal Discrete Cosine Transform (DCT). The advantage of using DCT as the initial dictionary is the low memory cost of storing the dictionary and faster learning due to the small size of the dictionary.

{\bf SET5:} Finally, we test the algorithm for the image inpainting problem. In this scenario, the measurements are acquired in the standard basis, i.e. a portion of pixels are measured. Pixel-wise sampling corresponds to the most practical acquisition approach. However, $\Phi_j$'s would not satisfy the sub-Gaussian concentration inequalities, resulting in a weaker performance than was promised in this paper. 

\subsection{Results and discussion}
The recovery PSNR results for every setting, averaged over $20$ trials for each case, are presented in Tables \ref{tab:set1} through \ref{tab:set5}. Since the block size is $n=64$, the sampling ratios $5\%$, $10\%$, $20\%$, $25\%$, $30\%$, $40\%$ and $50\%$ respectively correspond to $m=3$, $m=6$, $m=12$, $m=16$, $m=19$, $m=25$ and $m=32$ measurements per block. The running time on an Intel Core i5-3470 CPU with a clock speed of 3.2GHz, using Lasso on the Python platform developed for \cite{Mairal_2010}, respectively for $5\%$, $20\%$ and $50\%$ was measured around (on average) $40$, $50$ and $80$ seconds.

Generally speaking, DL-M results in noticeable block-CS recovery improvements for most images. We found that most cases would benefit from a higher $T_{\max}$ while some recovery results saturate or even degrade after a certain number of iterations between $10$ and $20$. This is expected since, as we showed in Section \ref{sec:math}, the image statistics can influence the DL-M performance and its resistance against overfitting. Not surprisingly, our best performance is achieved in {\bf SET1} where the image is guaranteed to allow exact sparse representation (although unknown to the solver). 

As the number of measurements per block is increased, the PSNR gain is more significant. This behavior is expected since, with more measurements, the DL-M problem becomes closer to the well-posed DL problem, i.e. DL based on the complete knowledge of the underlying image. Meanwhile, it is crucial that the PSNR gain stays positive for very low sampling ratios, e.g. for $m=3$ and $m=6$. As it can be seen, the proposed DL-M algorithm is reliable in every sampling ratio and resistive to overfitting for most test images. 

Finally, our inpainting results ({\bf SET5}) can be compared to the BCS results in \cite{bcs2} which uses the one-block-sparse signal model\footnote{Unfortunately, it is not possible to compare our results with the method of \cite{bcs3} since its code has not been released and the simulation settings and evaluation metrics in \cite{bcs3} are different than ours.}. In \cite{bcs2}, the recovery PSNR for Barbara's image with $n=64$, $p=256$ and $m=32$ was reported at $27.93$ dB which is lower than recovery PSNR using the non-adaptive universal dictionary reported in this paper. In addition to that, the proposed inpainting algorithm in \cite{bcs2} uses overlapping block-CS recovery which significantly increases $N$, resulting in performance boost at the cost of more computational complexity. Meanwhile, we have employed a non-overlapping framework to stay consistent with the general block-CS recovery where the measurement matrices can take any form and overlapping block recovery is not always possible.

\section{Conclusion}
\label{sec:discuss}
The analysis and the empirical experiments presented in this paper show that blind compressed sensing of natural images can be performed reliably without additional constraints on the dictionary or the sparse coefficients. Our results generalize the BCS frameworks in \cite{bcs1} and \cite{bcs2} by providing the probability of uniqueness, as well as accuracy. Meanwhile, we showed that dictionary learning is ill-posed when fixed block measurements are employed. The results of this paper can be extended to work with dense measurement matrices using (\ref{eq:24524552}). There are at least two main directions to improve and build upon the work reported in this paper:

\begin{itemize}
\item \textit{Deterministic sensing matrices}: We studied the class of random independent Gaussian sensing matrices for sensor design. Meanwhile, deterministic designs of sensing matrices are proven to perform equally well \cite{Devore_2007} while addressing the implementation constraints. 

\item \textit{Online DL-M}: It has been known \cite{Mairal_2010} that online gradient descent is more efficient than batch gradient descent for large data sets and especially when the data becomes available in a streaming fashion. Therefore, extension to online DL-M, both algorithmically and analysis-wise, could be the next step.

\end{itemize}

\end{document}